\documentclass{article}
\makeatletter
\input{arxiv.sty}
\makeatother
\usepackage[utf8]{inputenc}
\usepackage[T1]{fontenc}
\usepackage[authoryear]{natbib}
\usepackage[hypertexnames=false]{hyperref}
\hypersetup{pdftitle={Certified Behavioral Retention for Continual Adaptation with Held-Out Anchors},pdfauthor={Ahanaf Hasan Ariq and Abrar Shahid}}
\usepackage{url}
\usepackage{booktabs}
\usepackage{amsfonts}
\usepackage{amsmath,amssymb}
\usepackage{amsthm}
\usepackage{nicefrac}
\usepackage{microtype}
\usepackage{xcolor}
\usepackage{algorithm}
\usepackage{algorithmic}
\usepackage{tikz}
\definecolor{methodblue}{HTML}{4E79A7}
\definecolor{methodteal}{HTML}{2A9D8F}
\definecolor{methodorange}{HTML}{E07A3F}
\definecolor{methodgray}{HTML}{6B7280}
\setcitestyle{authoryear,square,semicolon,aysep={,},yysep={;}}
\newtheorem{assumption}{Assumption}
\newtheorem{theorem}{Theorem}
\newtheorem{corollary}{Corollary}
\newtheorem{proposition}{Proposition}
\title{Certified Behavioral Retention for Continual Adaptation with Held-Out Anchors}
\author{%
  Ahanaf Hasan Ariq \\
  Ideal School and College \\
  \And
  Abrar Shahid \\
  Notre Dame College \\
}
\date{}
\begin{document}
\maketitle
\begin{abstract}
Continual adaptation requires deciding not only how to update a model, but whether to commit the update without exceeding a declared budget for changes to protected behavior. We introduce Certified Anchor-Based Adaptation (CABA), a candidate-agnostic post-update gate that reserves historical observations from candidate optimization, compares old and candidate policies on frozen representations, and selects the largest tested interpolation coefficient whose worst-anchor forward KL divergence does not exceed the budget. The finite retention check requires neither action labels nor complete trajectories. We develop a theory-led certification framework that separates measured anchor retention from conditional off-anchor guarantees. Under explicit representation-coverage and Lipschitz assumptions, an accepted update satisfies a uniform behavioral-divergence bound on protected support. The analysis provides regularity bounds for affine softmax heads, high-probability coverage and compression guarantees, interpolation-path monotonicity conditions, grid-resolution bounds, and trajectory-level and sequential-composition guarantees. These results specify the assumptions needed to interpret a finite acceptance check; they do not assert global safety or preserved task accuracy. In a controlled UCI Human Activity Recognition study used as a non-stationary sensor-policy proxy, CABA reduces mean forgetting from 1.01 percentage points for unconstrained adaptation to 0.00 at budget $\epsilon=0.05$, while reducing final average accuracy from 93.04\% to 88.48\%. This proof of concept illustrates a retention--plasticity trade-off; it does not establish foundation-model or robotic performance. CABA supplies an auditable retention check for compatible candidate generators exposing comparable policy distributions and a declared update path, while deployment additionally requires separate evidence of current-task utility.
\end{abstract}

\section{Introduction}
A foundation-model agent is increasingly a changing system rather than a fixed artifact. New tools, feedback, users, sensors, and environments create repeated opportunities for post-training, model editing, memory revision, or adapter updates. Each update can improve the current objective while degrading behavior that was already relied upon. This stability--plasticity problem is central to continual learning \citep{wang2024survey,delange2022survey}; for deployed agents, it is also a change-management problem. Before a new model, adapter, prompt, memory, or policy is committed, an operator needs evidence about what changed.

Replay \citep{lopezpaz2017gem,smith2024replay}, parameter or subspace regularization \citep{kirkpatrick2017ewc,saha2021gpm}, and trust regions \citep{schulman2015trpo} intervene during candidate optimization. Their retention objectives are valuable, but success on historical examples used for training is different evidence from performance on a held-out acceptance set. CABA makes that distinction explicit without replacing those training methods.

We study an interface in which an arbitrary update generator proposes a candidate and a separate gate decides how much of it to commit. The gate retains historical observations or frozen representations excluded from candidate optimization. It measures behavioral divergence between the current and interpolated candidate policies, takes the maximum over anchors rather than an average, and accepts the largest tested interpolation step below a predeclared budget. We call this wrapper Certified Anchor-Based Adaptation (CABA).

CABA is an acceptance mechanism, not a complete continual-learning algorithm. It can wrap supervised, reinforcement-learning, LoRA \citep{hu2022lora}, editing, or replay candidates that expose comparable policy distributions and a declared update path. Discrete prompt or tool changes require a finite candidate path, not literal parameter interpolation. Small KL does not imply utility, fairness, security, or global safety. A deployment needs two keys: current-task utility and protected-behavior retention. CABA supplies the latter and records the budget, anchor and candidate identifiers, tested path, accepted coefficient, maximum divergence, and worst anchor.

The closest recent systems make the novelty boundary important. STABLE gates continual LoRA edits using exact-match loss, confidence change, or KL and may rescale or reject a merge \citep{hoy2025stable}. Harness Continual Learning separates proposal from commitment for evolving prompts, memories, tools, and routers \citep{kang2026harness}. We do not claim that separation or label-free KL alone is novel. Our contribution is the formalization of a held-out worst-anchor contract and the assumptions under which its finite check extends beyond stored anchors.

\section{Problem Setup}
Let $o$ be an observation, $z=\phi(o)$ a representation, and $\pi_\theta(\cdot\mid z)$ a stochastic policy or output distribution. At update $t$, the learner receives batch $D_t$ and partitions it before candidate training:
\begin{equation}D_t=D_t^{\mathrm{train}}\,\dot\cup\,D_t^{\mathrm{anc}}.\end{equation}
An arbitrary adaptation routine uses $D_t^{\mathrm{train}}$ and permitted non-anchor resources to produce candidate $\widehat\theta_t$ from current accepted parameters $\theta_{t-1}$. The gate is then evaluated on accumulated held-out anchors
\begin{equation}A_{<t}=\bigcup_{s<t}\{\phi(o):o\in D_s^{\mathrm{anc}}\}.\end{equation}
Action labels are not required for the retention test because both policies can be queried on the same representation. Labels may be stored for a replay baseline or separate utility test, but are not inputs to the CABA certificate.

For old parameters $\theta$ and candidate parameters $\theta'$, define forward KL in nats:
\begin{equation}d(z;\theta,\theta')=D_{\mathrm{KL}}\!\left(\pi_\theta(\cdot\mid z)\,\middle\|\,\pi_{\theta'}(\cdot\mid z)\right).\end{equation}
Forward KL penalizes probability removed from actions supported by the current policy. Other divergences can be substituted, but the trajectory result below uses a divergence controlling total variation. Recent work relates policy KL to forgetting in continual post-training while also showing that equal KL values need not imply equal task accuracy \citep{shenfeld2025razor,hoy2025stable}. CABA treats KL as a change budget, not as a direct performance metric.

\paragraph{Threat model.} CABA targets accidental behavioral drift caused by a candidate update. It assumes the old and candidate policies can be evaluated faithfully on protected anchors, the anchor store has integrity, and the declared divergence is meaningful for the action interface. Holding anchors out reduces direct training-set reuse, but repeated accept/reject feedback still conveys anchor information. The finite check does not defend against a malicious candidate that detects auditing, compromised evaluation code, failures shared by both policies, or hazards outside protected support. The mathematical certificate is conditional on coverage and regularity; the empirical anchor check alone is not a global-safety guarantee.

\paragraph{Protected behavior and current utility.} Let $P_t$ be the protected historical representation distribution and $U_t(\theta)$ a current-task utility. A deployment rule should require both a retention condition and a separately evaluated utility condition, for example
\begin{equation}\max_{z\in A_{<t}}d(z;\theta_{t-1},\theta_t)\leq\epsilon_t,\qquad \operatorname{LCB}\{U_t(\theta_t)-U_t(\theta_{t-1})\}\geq\nu_t,\end{equation}
where the lower confidence bound uses fresh utility data after path selection, or simultaneous coverage over the tested path. Our experiment isolates the retention key; this equation states the deployment interface.

\section{CABA: A Candidate-Agnostic Commit Gate}
Let $\Gamma\subset[0,1]$ be a finite interpolation grid containing $0$. For adapter or policy-head parameters, define
\begin{equation}\theta_t(\alpha)=(1-\alpha)\theta_{t-1}+\alpha\widehat\theta_t.\end{equation}
CABA computes the worst-anchor divergence
\begin{equation}g_t(\alpha)=\max_{z\in A_{<t}}d(z;\theta_{t-1},\theta_t(\alpha))\end{equation}
and selects
\begin{equation}\alpha_t=\max\{\alpha\in\Gamma:g_t(\alpha)\leq\epsilon_t\},\qquad \theta_t=\theta_t(\alpha_t).\end{equation}
We define the empty-anchor score as zero and require $\epsilon_t\geq0$. The first candidate passes the retention key only; deployment also needs the utility key. If no nonzero grid point is feasible, the update is rejected. Exhaustive grid evaluation needs no monotonicity assumption. With a verified numerical evaluation-error bound $e_{\mathrm{num}}$, accept only if $\widehat g_t(\alpha)+e_{\mathrm{num}}\leq\epsilon_t$.

\paragraph{Why worst-anchor control?} An average divergence can remain small while one protected state changes sharply. The maximum makes every stored anchor a hard constraint and identifies the anchor closest to rejection. This is conservative and sensitive to corrupted or irrelevant anchors; a high-quantile alternative is a different contract rather than a silent substitute.

\paragraph{Why hold anchors out?} The finite check is deterministic once the anchors and evaluated policies are fixed. Excluding audit anchors from candidate optimization permits a disjoint replay buffer, but repeated accept/reject feedback still conveys information about anchors. Held-out therefore does not mean statistically independent across updates; stochastic generalization claims need fresh audits or an adaptive-valid analysis.

\paragraph{Anchor lifecycle.} After deciding update $t$, representations from $D_t^{\mathrm{anc}}$ are eligible for future gates. A bounded system can use reservoir sampling, coverage-aware coresets, or safety-prioritized anchors. Hindsight Anchor Learning optimizes anchors for training \citep{chaudhry2021hindsight}; Memory Anchors selects high-conflict robot experiences for replay \citep{du2026memory}. CABA can use their selection insights, but the audit copy remains excluded from candidate optimization. Deletion, refresh, and access policies must be logged because changing the anchor set changes the retention claim.

\begin{algorithm}[t]\caption{CABA commit gate with an audit receipt}\label{alg:caba}\begin{algorithmic}[1]
\REQUIRE Current parameters $\theta$, candidate $\widehat\theta$, anchors $A$, grid $\Gamma$, budget $\epsilon$
\STATE Initialize feasible set $F\leftarrow\{0\}$ and receipt $R\leftarrow\varnothing$
\FOR{each $\alpha\in\Gamma\setminus\{0\}$}
\STATE $\theta(\alpha)\leftarrow(1-\alpha)\theta+\alpha\widehat\theta$
\STATE Evaluate $q_\alpha(z)\leftarrow D_{\mathrm{KL}}(\pi_\theta(\cdot\mid z)\|\pi_{\theta(\alpha)}(\cdot\mid z))$ for $z\in A$
\STATE Record $g(\alpha)\leftarrow\max_{z\in A}q_\alpha(z)$ and a maximizing anchor; use $0$ and null if $A=\varnothing$
\IF{$g(\alpha)\leq\epsilon$}\STATE $F\leftarrow F\cup\{\alpha\}$\ENDIF
\ENDFOR
\STATE $\alpha^\star\leftarrow\max F$; append model, anchor, grid, budget, and score hashes to $R$
\RETURN Committed parameters $\theta(\alpha^\star)$ and receipt $R$
\end{algorithmic}\end{algorithm}

\section{Conditional Certification Results}
The results separate the finite computation performed by Algorithm~\ref{alg:caba} from assumptions required to extend it beyond stored anchors.
\begin{assumption}[Representation coverage] For protected distribution $P$ with support $S\subset\mathcal Z$, the finite nonempty anchor set $A\subset S$ is a $\rho$-net: every $z\in S$ has $z_i\in A$ with $\|z-z_i\|_2\leq\rho$.\end{assumption}
\begin{assumption}[Divergence regularity] For accepted pair $(\theta,\theta')$, $z\mapsto d(z;\theta,\theta')$ is $L$-Lipschitz on $S$.\end{assumption}
\begin{theorem}[Finite-anchor retention certificate] If CABA accepts $\theta'$ with $\max_{z_i\in A}d(z_i;\theta,\theta')\leq\epsilon$ and the assumptions hold, then
\begin{equation}\sup_{z\in S}d(z;\theta,\theta')\leq\epsilon+L\rho.\end{equation}\end{theorem}
\begin{proof}For any $z\in S$, choose a covering anchor $z_i$. Acceptance bounds $d(z_i;\theta,\theta')$ by $\epsilon$, and Lipschitz continuity adds at most $L\|z-z_i\|_2\leq L\rho$.\end{proof}

The theorem does not establish coverage from random sampling alone. Sparse anchors and sharp policies increase $L\rho$. For affine softmax logits $(W_jz+b_j)/T_{\mathrm{temp}}$, $j=0,1$, on $\|z\|_2\leq R$ with $T_{\mathrm{temp}}>0$, a sufficient constant is
\begin{equation}L\leq\frac{(\|W_0\|_2+\|W_1\|_2)(\|W_0-W_1\|_2R+\|b_0-b_1\|_2)}{2T_{\mathrm{temp}}^2}.\end{equation}
Here matrix norms are spectral. Appendix~\ref{app:regularity} derives this computable bound and its Jacobian extension. Neither a maximum sampled gradient nor a sampled radius certifies a support-wide bound. No empirical $L$ estimate is claimed.

\begin{theorem}[High-probability coverage] Fix $P$, $\rho>0$, and a finite cover of $S$ by $N$ balls $B(c_j,\rho/2)$ before sampling. Suppose each ball has mass at least $\mu_\rho>0$. For $m$ independent anchors drawn from $P$,
\begin{equation}\Pr(A\text{ is a }\rho\text{-net})\geq1-Ne^{-m\mu_\rho}.\end{equation}
For $\delta\in(0,1)$, $m\geq\mu_\rho^{-1}\log(N/\delta)$ suffices for coverage with probability at least $1-\delta$. On this event, every accepted pair satisfying divergence regularity obeys the finite-anchor certificate, including pairs selected using these anchors.\end{theorem}
Rare hazards may have tiny mass and need deliberate anchors. In a bounded doubling space, covering numbers can grow as $(D/\rho)^q$, making uniform coverage expensive. A fixed-capacity or conflict-selected store does not automatically satisfy the independent-sampling premise; its coverage must be justified separately.

\begin{corollary}[Compressed anchors] Under coverage, store $\widetilde z_i$ with $\|z_i-\widetilde z_i\|_2\leq\eta$ and verify $d(\widetilde z_i;\theta,\theta')\leq\epsilon$. Require $L$-Lipschitz regularity on $S\cup\{\widetilde z_i\}$, not only on $S$. Then
\begin{equation}\sup_{z\in S}d(z;\theta,\theta')\leq\epsilon+L(\rho+\eta).\end{equation}\end{corollary}
Thus storing compressed representations instead of raw observations has an explicit price. Quantization or projection is not free merely because the gate remains executable.

\begin{proposition}[Grid resolution] Let $g$ be finite, continuous, and nondecreasing on $[0,1]$, with $g(0)=0$ and $\epsilon\geq0$. Let $\Gamma=\{0=\gamma_0<\cdots<\gamma_M=1\}$ and $h=\max_j(\gamma_{j+1}-\gamma_j)$. For $\alpha_{\mathrm{cont}}=\max\{\alpha\in[0,1]:g(\alpha)\leq\epsilon\}$, the feasible grid choice satisfies $\alpha_t\geq\max\{0,\alpha_{\mathrm{cont}}-h\}$.\end{proposition}
Without monotonicity, exhaustive grid evaluation remains valid, but binary search and this approximation need not be. For logits affine in $\alpha$ at each anchor, forward KL is convex in $\alpha$, minimized at zero, hence nondecreasing. This covers interpolation of a linear softmax head, but not general LoRA factor interpolation; Appendix~\ref{app:regularity} proves this distinction.

\begin{proposition}[Bounded-horizon behavior degradation] Let $J_H(\pi)=\mathbb E_\pi[\sum_{h=0}^{H-1}r_h]$ for $H\geq1$ in the same MDP with common initial, transition, and reward laws and $r_h\in[0,1]$. If the uniform bound holds with $\kappa=\epsilon+L\rho$ at every state reachable under the old policy, then
\begin{equation}\left|\frac{J_H(\pi_{\theta'})}{H}-\frac{J_H(\pi_\theta)}{H}\right|\leq\min\left\{1,\sqrt{\frac{H\kappa}{2}}\right\}\leq\min\left\{1,H\sqrt{\frac{\kappa}{2}}\right\}.\end{equation}
The KL chain rule and Pinsker give the first bound; stepwise maximal coupling gives the weaker second bound. Both require coverage of relevant trajectories, not merely observed prompts or states. The sharper bound can still be vacuous.\end{proposition}
\begin{corollary}[Sequential composition] For accepted policies $\pi_0,\ldots,\pi_T$ in the same protected MDP, let $\kappa_t$ be the valid uniform KL bound for update $t$. Then
\begin{equation}\left|\frac{J_H(\pi_T)-J_H(\pi_0)}{H}\right|\leq\min\left\{1,\sum_{t=1}^T\sqrt{\frac{H\kappa_t}{2}}\right\}.\end{equation}\end{corollary}
The additive bound is conservative: individually acceptable changes can accumulate, so a continual system must track per-update and cumulative budgets and periodically compare against a stable reference policy.

\section{Relationship to Continual-Learning Approaches}
\begin{table}[t]\centering\caption{Relationship to representative retention mechanisms. ``Post-candidate'' means retention is checked after candidate generation.}\label{tab:related}\small\begin{tabular}{p{0.9in}p{1.2in}p{1.1in}p{1.6in}}\toprule Approach & Retained information & Role in update & Primary retention evidence\\\midrule GEM / replay & Labeled past examples & Candidate training & Gradient or rehearsal performance\\Hindsight anchors & Learned labeled anchors & Candidate training & Empirical forgetting reduction\\TRPO / LCPO & Current or old policy states & Candidate optimization & Trust-region or method-specific analysis\\STABLE & Anchor questions and metrics & Post-candidate LoRA gate & Empirical rescale/reject behavior\\Memory Anchors & Labeled conflict experiences & Replay selection & Robot-policy retention\\CABA & Held-out observations or representations & Post-candidate commit gate & Worst-anchor test plus conditional certificate\\\bottomrule\end{tabular}\end{table}

\paragraph{Replay and regularization.} GEM constrains gradients using episodic memory \citep{lopezpaz2017gem}; Gradient Projection Memory protects representation-derived subspaces \citep{saha2021gpm}; adaptive replay policies allocate limited compute across remembered examples \citep{smith2024replay}. These methods seek a better candidate. CABA asks whether a candidate produced by any such method satisfies an independently declared contract.
\paragraph{Policy constraints and continual post-training.} TRPO introduced KL trust regions for stable policy improvement \citep{schulman2015trpo}. LCPO anchors outputs on old experiences in non-stationary environments \citep{hamadanian2025lcpo}. Recent CPO, MCPO, and CRL-VLA incorporate behavioral divergence into continual policy or VLA optimization \citep{luo2026rlforgets,liao2026mcpo,zeng2026crlvla}. RL's Razor connects smaller policy KL with reduced forgetting across language and robotic foundation models \citep{shenfeld2025razor}. These works motivate behavioral divergence, but do not themselves make a finite held-out post-update check a uniform certificate.
\paragraph{Gated updates and evolving agent state.} STABLE separates edit generation from acceptance and gates LoRA merges using exact-match drop, bits increase, or average KL \citep{hoy2025stable}. CABA does not claim this separation or label-free KL as novel; its distinction is a held-out worst-anchor contract with explicit off-anchor assumptions. Harness Continual Learning covers evolving prompts, memories, tools, and routers \citep{kang2026harness}. No comparative performance advantage is established.
\paragraph{Anchor-based robot learning.} Memory Anchors uses representation overlap and action conflict to enrich robot replay \citep{du2026memory}. Conflict-aware audit selection is complementary, but its benefit for CABA remains untested and the audit copy must remain separate from training.

\section{Controlled Study on UCI-HAR}
The experiment evaluates whether the gate exposes a measurable retention--plasticity trade-off. It is not evidence about a language model, VLA policy, or physical robot.
\paragraph{Dataset and stream.} We use the official Human Activity Recognition Using Smartphones dataset, containing inertial features from 30 subjects performing six activities \citep{anguita2013har}. We retain the official train/test partition, standardize using training statistics only, and form five chronological subject ranges $[1,5]$, $[6,10]$, $[11,15]$, $[16,20]$, and $[21,25]$. Updates use official-training examples in the current range; evaluation uses official-test examples in the corresponding range. Subjects outside these ranges are not reported. The 561-dimensional feature vector is treated as a frozen representation and a linear softmax head as the policy.
\paragraph{Anchor protocol.} Each incoming training group is split into candidate-training and held-out anchor portions. After its update decision, the group contributes at most 256 representations for future retention checks, for a maximum of 1,280 anchors. Candidate training never uses CABA audit anchors.
\paragraph{Methods.} Finetune commits the candidate head after each group. Replay trains on current examples plus stored anchor representations and labels. Frozen trains the initial head and rejects subsequent changes. CABA uses the same candidate as Finetune, evaluates $\Gamma=\{0,0.05,\ldots,1\}$, and commits the largest interpolation satisfying worst-anchor budget $\epsilon=0.05$. All adaptive methods use matched optimization settings and seeds $\{0,1,2\}$.
\paragraph{Metrics.} Let $A_{t,k}$ be test accuracy on group $k$ after update $t$. For $T=5$, final average accuracy is $T^{-1}\sum_{k=1}^T A_{T,k}$, and average forgetting is
\begin{equation}F_T=\frac1T\sum_{k=1}^T\left(\max_{t\in\{k,\ldots,T\}}A_{t,k}-A_{T,k}\right).\end{equation}
Forgetting is nonnegative by construction and measures loss relative to each group's best observed accuracy. It is not a safety metric and can be zero for a model that never learns, so it must be read together with final accuracy.
\begin{figure}[t]
\centering
\begin{tikzpicture}[x=1cm,y=1cm,font=\footnotesize]
\begin{scope}
\node[anchor=west,font=\bfseries] at (0,4.72) {Final average accuracy (\%)};
\foreach \x/\tick in {1.30/0,2.75/25,4.20/50,5.65/75,7.10/100} {
  \draw[gray!25] (\x,0.48)--(\x,4.08);
  \node[below,text=gray!70] at (\x,0.35) {\tick};
}
\draw[gray!65] (1.30,0.48)--(7.10,0.48);
\foreach \row/\method/\value/\barcolor in {
  3.62/Finetune/93.04/methodblue,
  2.68/Replay/93.17/methodteal,
  1.74/CABA/88.48/methodorange,
  0.80/Frozen/84.42/methodgray} {
  \pgfmathsetmacro{\barwidth}{0.058*\value}
  \node[anchor=east] at (1.18,\row) {\method};
  \path[fill=\barcolor,rounded corners=1.5pt]
    (1.30, {\row-0.17}) rectangle ({1.30+\barwidth}, {\row+0.17});
  \node[anchor=west] at ({1.38+\barwidth},\row) {\value};
}
\end{scope}
\begin{scope}[xshift=8.10cm]
\node[anchor=west,font=\bfseries] at (0,4.72) {Average forgetting (pp)};
\foreach \x/\tick in {1.30/0,2.57/0.4,3.83/0.8,5.10/1.2} {
  \draw[gray!25] (\x,0.48)--(\x,4.08);
  \node[below,text=gray!70] at (\x,0.35) {\tick};
}
\draw[gray!65] (1.30,0.48)--(5.10,0.48);
\foreach \row/\method/\value/\barcolor in {
  3.62/Finetune/1.01/methodblue,
  2.68/Replay/0.85/methodteal,
  1.74/CABA/0.00/methodorange,
  0.80/Frozen/0.00/methodgray} {
  \pgfmathsetmacro{\barwidth}{3.1667*\value}
  \pgfmathsetmacro{\visiblewidth}{max(\barwidth,0.05)}
  \node[anchor=east] at (1.18,\row) {\method};
  \path[fill=\barcolor,rounded corners=1.5pt]
    (1.30, {\row-0.17}) rectangle ({1.30+\visiblewidth}, {\row+0.17});
  \node[anchor=west] at ({1.38+\visiblewidth},\row) {\value};
}
\end{scope}
\end{tikzpicture}
\caption{Three-seed UCI-HAR results. Bars show means; Table~\ref{tab:results} reports sample standard deviations. Accuracy uses the full 0--100\% scale, and lower forgetting is better. CABA reaches zero measured forgetting but trades off final accuracy; Frozen shows why forgetting alone is insufficient.}\label{fig:results}
\end{figure}
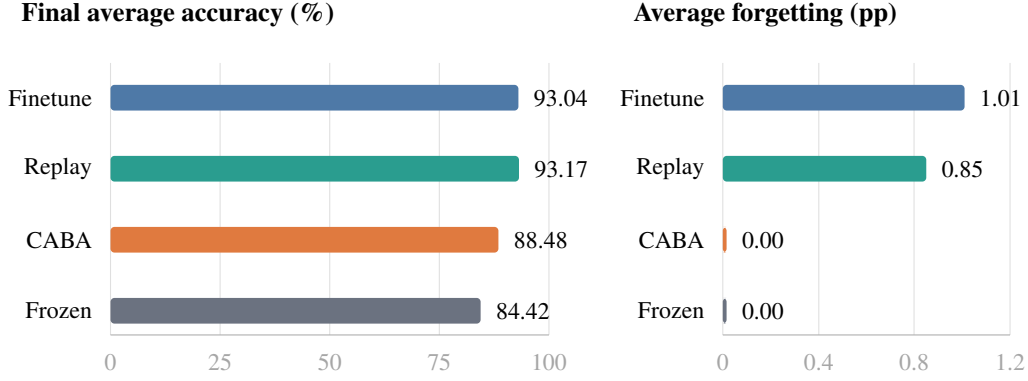
\begin{table}[t]\centering\caption{UCI-HAR continual adaptation. Mean $\pm$ sample standard deviation over three seeds. Higher accuracy and lower forgetting are preferred.}\label{tab:results}\begin{tabular}{lcc}\toprule Method & Final average accuracy $\uparrow$ & Forgetting (pp) $\downarrow$\\\midrule Finetune & $93.04\pm0.40$ & $1.01\pm0.39$\\Replay & $93.17\pm0.29$ & $0.85\pm0.25$\\CABA ($\epsilon=0.05$) & $88.48\pm1.14$ & $0.00\pm0.00$\\Frozen & $84.42\pm0.82$ & $0.00\pm0.00$\\\bottomrule\end{tabular}\end{table}
\paragraph{Results.} Finetune and Replay reach about 93\% final average accuracy but forget 1.01 and 0.85 percentage points. CABA reduces measured forgetting to zero under the fixed anchor budget while reaching 88.48\% accuracy, 4.56 points below Finetune and 4.75 below Replay. Frozen also has zero forgetting but reaches only 84.42\%. The appropriate conclusion is not that CABA dominates replay. The gate converts retention into an explicit constraint and exposes how much adaptation the declared constraint blocks.
\paragraph{Budget diagnostic.} A stricter operating point, $\epsilon=0.02$, yields 86.17\% final accuracy and 0.02 percentage points average forgetting in the same three-seed protocol. Budget selection must occur before inspecting final test performance; a sweep is sensitivity analysis, not permission to select a favorable test result after the fact.
\paragraph{What the experiment validates.} The study validates the executable finite-anchor gate and retention--plasticity control on a real sensor dataset. It does not validate coverage, estimate $\rho$ or $L$, or establish the uniform guarantee. It does not show that observation-only storage is private, that a large agent cannot overfit an audit distribution, or that small action KL preserves semantic task success.

\section{Deployment Design for Foundation-Model Agents}
For an LLM or VLA, the audited output can be a next-token, tool-selection, or action distribution with computable divergence. Internal LoRA changes generally change downstream features despite frozen base weights: store raw contexts and evaluate both complete policies. A genuinely frozen feature map permits cached representations. The guarantee concerns the audited stochastic policy, not a different greedy decoder or untested tool execution.
\paragraph{Evaluation harness.} An agent-improvement loop has proposal, training, held-out evaluation, and commitment stages. CABA belongs between training and commitment. The candidate-generation service receives no audit-anchor identifiers. The evaluator loads current and candidate adapters, computes every declared grid point, checks the current-utility key, signs the receipt, and updates the serving pointer only after both keys pass.
\paragraph{Scaling cost.} With $m$ anchors, $G=|\Gamma|$, and evaluation cost $C_\pi$, exhaustive gating costs $O(mGC_\pi)$; old-policy logits can be cached. A proven monotone path permits $O(mC_\pi\log G)$ grid search, but sampled monotonicity alone is insufficient. A coreset trades memory against coverage radius. These are operation counts, not measured runtime results.
\paragraph{Anchor selection and refresh.} Random anchors approximate frequent behavior and may miss rare hazards. A defensible store combines distributional samples, known safety-critical states, task-boundary cases, representation-coverage points, and conflict points. Refresh anchors when sensors, prompts, tools, or the frozen backbone change, because the prior coverage statement may no longer concern the deployed representation space.
\paragraph{Beyond one-step retention.} Maintain drift from the immediately previous policy, drift from a stable reference, and task-level performance on historical evaluations. Per-update CABA controls the first; cumulative comparisons and semantic regression tests remain necessary.

\section{Discussion: Scope of Evidence and Evaluation}
The finite-anchor result is the empirically supported part of CABA in this paper. We do not estimate a coverage radius or support-wide Lipschitz constant on UCI-HAR: the experiment reports a finite held-out check, and empirical maxima over observed features would not certify either quantity on full protected support. Reporting a numerical value for $\epsilon+L\rho$ without defensible support bounds would overstate the evidence. A synthetic example could illustrate the theorem under stipulated assumptions, but would not establish those assumptions for UCI-HAR or an agent deployment.

The exhaustive cost can be made concrete without presenting it as a runtime benchmark. The anchor cap is 1,280 and the grid has 21 points including $\alpha=0$. Since zero is known feasible without policy inference, evaluating all 20 nonzero points performs at most 25,600 anchor--coefficient divergence evaluations at that cap. For the six-class, 561-feature linear head, computing candidate distributions at every pair is on the order of 86 million multiply-accumulates, before batching or implementation-specific optimizations. This is a derived operation-count upper bound, not measured wall-clock time. Caching old-policy logits reduces repeated work. A verified monotone path permits binary search; otherwise skipping grid points changes the algorithm's guarantee.

No same-candidate comparison against a STABLE-style gate is reported. Such a comparison should hold the old policy, candidate generator and candidate, interpolation path, audit data, and utility test fixed, then compare acceptance decisions and costs. Comparing methods that train different candidates would confound candidate quality with gate behavior. Appendix~\ref{app:protocol} separates these factors.

The worst-anchor contract is intentionally sensitive: a stale or corrupted anchor may reject a useful candidate, while compromised anchor representations or evaluator integrity can invalidate the measured claim. Adversarially chosen anchors can leave important support uncovered. A high-quantile gate could improve tolerance to isolated outliers, but it no longer guarantees that every retained anchor passes; a robustness claim needs a predeclared contamination model, authenticated anchor provenance, and a separate coverage argument. We report no anchor-corruption or quantile-gating experiment.

\paragraph{Prospective robustness and gate comparison.} A follow-up should first freeze a candidate generator and produce one candidate per update, then pass that identical candidate, interpolation path, old policy, and audit set through unrestricted commit, a STABLE-style gate, and CABA. The primary paired outcomes would be accepted coefficient, worst-anchor and average KL, current-task utility, and evaluator cost; an end-to-end stream would then measure how different commit decisions alter later candidates. This distinguishes the gate's effect from differences in candidate training.

Anchor sensitivity should be reported separately for stale-but-valid anchors, corrupted representations, and deliberate contamination. A simple robust empirical contract could permit at most $b$ of $m$ anchors to exceed budget, $\sum_{i=1}^{m}\mathbf{1}[d(z_i;\theta,\theta')>\epsilon]\leq b$. This changes the guarantee: it bounds the number of observed violations, not the maximum divergence over anchors or protected support. Any selection, quarantine, or replacement rule must be fixed before final evaluation and accompanied by a coverage analysis for the retained set. Randomly perturbing anchor features would test evaluator sensitivity but would not model adversarially selected states; these should not be conflated.

For UCI-HAR, a stronger certificate study would need a justified domain/support definition, a defensible covering construction, and a domain-wide Jacobian bound for the evaluated policy pair. Empirical nearest-neighbor radii, sampled gradient maxima, and their resulting plug-in values can be reported as diagnostics, but not as the uniform certificate unless accompanied by valid upper bounds. For an LLM/VLA evaluation, the same paired design should additionally hold prompt or observation histories fixed, use an independently measured utility set, version the complete policy and representation, and report both distributional drift and task outcomes. These measurements remain future work; no numerical results are implied here.

For an LLM, an audit item should identify a versioned conversation history and compare next-token or tool-selection distributions under complete old and candidate systems. For a VLA, it should identify a versioned observation--instruction history and compare action distributions; trajectory outcomes and safety events require separate evaluation. If a backbone or feature extractor changes, raw-context reevaluation or a certified representation-error allowance is needed. In both cases, current utility must be measured on separate data and paired with an independently protected retention test. These are deployment requirements, not results established by the UCI-HAR proxy.

\section{Limitations, Broader Impact, and Conclusion}
\paragraph{Limitations.} Coverage is the main limitation: random anchors can miss rare states, and the affine-softmax regularity bound may be loose. Neither $\rho$ nor $L$ is estimated in the UCI-HAR study; its finite checks do not instantiate the uniform certificate. Long-horizon bounds can be vacuous, nonlinear interpolation can cross high-drift regions, KL is asymmetric, and representations can reveal sensitive information. Evidence remains limited to one linear head, five subject groups, and three seeds. No foundation-model or embodied validation, direct STABLE comparison, or anchor-corruption experiment is claimed.
\paragraph{Broader impact.} A separate commit gate can make continual updates easier to audit, compare, reject, and roll back. It can also create false confidence if the anchor set is narrow, the evaluator is compromised, or ``certified'' is reported without assumptions. Report finite checks and off-anchor assumptions separately. Protected anchors must be governed as sensitive deployment data, and a retention gate must not prevent necessary corrections to harmful old behavior; explicit exceptions and human authorization are required.
\paragraph{Conclusion.} CABA formalizes a retention contract after candidate generation. Its contribution is the explicit connection between finite checks, coverage, regularity, path choice, and cumulative drift; the underlying inequalities are standard. The controlled study illustrates retention's accuracy cost, not foundation-model effectiveness. A future LoRA or VLA study should compare post-update gates on shared candidates, then vary training-time KL, replay, and anchor selection separately. Until then, CABA is a theory-led framework rather than an empirically superior continual learner.

\appendix
\section{Proofs and Additional Derivations}\label{app:proofs}
A ball $B_j$ is missed with probability $(1-P(B_j))^m\leq e^{-m\mu_\rho}$. A union bound gives simultaneous occupancy probability at least $1-Ne^{-m\mu_\rho}$. For any $z\in S$, its covering ball contains anchor $z_i$ with $\|z-z_i\|_2\leq\|z-c_j\|_2+\|c_j-z_i\|_2\leq\rho$. Thus $A$ is a $\rho$-net. Solving $Ne^{-m\mu_\rho}\leq\delta$ proves the sample bound.

Under anchor compression, $\|z-\widetilde z_i\|_2\leq\|z-z_i\|_2+\|z_i-\widetilde z_i\|_2\leq\rho+\eta$. The gate bounds $d(\widetilde z_i)$ by $\epsilon$ and Lipschitz continuity on the enlarged domain adds $L(\rho+\eta)$. Regularity only on $S$ does not suffice if $\widetilde z_i\notin S$.

\subsection{Computable regularity bounds}\label{app:regularity}
Let scaled logits $u(z),v(z)$ be continuously differentiable on a convex domain containing protected and compressed points. Set $p=\operatorname{softmax}(u)$, $q=\operatorname{softmax}(v)$, and $C(p)=\operatorname{diag}(p)-pp^\top$. Log-sum-exp differentiation \citep{gao2017softmax} gives
\begin{equation}\nabla_zD_{\mathrm{KL}}(p\|q)=(Du)^\top C(p)(u-v)+(Dv)^\top(q-p).\end{equation}
For unit $x$, the variance $x^\top C(p)x\leq(\max_i x_i-\min_i x_i)^2/4\leq1/2$. Hence $\|C(p)\|_2\leq1/2$ and softmax is $1/2$-Lipschitz. Domain-wide bounds $\|Du\|_2\leq J_0$, $\|Dv\|_2\leq J_1$, and $\|u-v\|_2\leq B$ yield $L\leq(J_0+J_1)B/2$. Affine substitution proves the main-text bound; smaller temperature increases it. Nonlinear heads need certified domain-wide Jacobian bounds. Sampled gradient maxima are diagnostics, not upper certificates, and do not establish coverage.

\paragraph{Monotone paths.} For fixed logits $u,v$, let $\Delta=v-u$, $F(w)=\log\sum_i e^{w_i}$, and $q_\alpha=\operatorname{softmax}(u+\alpha\Delta)$. Then
\begin{equation}D_{\mathrm{KL}}(q_0\|q_\alpha)=F(u+\alpha\Delta)-F(u)-\alpha q_0^\top\Delta,\quad\frac{d^2}{d\alpha^2}D_{\mathrm{KL}}(q_0\|q_\alpha)=\operatorname{Var}_{q_\alpha}(\Delta)\geq0.\end{equation}
The derivative is zero at $\alpha=0$, so each score and their maximum are nondecreasing. This covers linear heads, not general LoRA factors. A policy mixture $(1-\alpha)p+\alpha q$ also gives convex, nondecreasing forward KL, bounded by $-\log(1-\alpha)$ since the mixture is at least $(1-\alpha)p$. It changes deployment semantics and may require both models; no mixture experiment is claimed.

\paragraph{Trajectory and sequential bounds.} For trajectory laws $P,Q$, common initial, transition, and reward kernels cancel. The KL chain rule gives $D_{\mathrm{KL}}(P\|Q)=\sum_{h=0}^{H-1}\mathbb{E}_P d(z_h;\theta,\theta')\leq H\kappa$. Pinsker bounds total variation by $\sqrt{H\kappa/2}$. Alternatively, per-state Pinsker and maximal coupling give disagreement probability at most $\sqrt{\kappa/2}$ per step; a union bound proves the weaker result. Autoregressive certification must cover relevant histories, not just initial prompts. The triangle inequality bounds normalized endpoint return difference by $\sum_t|J_H(\pi_t)-J_H(\pi_{t-1})|/H\leq\sum_t\min\{1,\sqrt{H\kappa_t/2}\}$. This sums return differences, not KL divergences.

\section{Reproducibility and Evaluation Receipt}\label{app:receipt}
An accompanying artifact should expose fields rather than only aggregate scores. The study fixes the official partition, subject ranges, anchor capacity, grid, KL direction, budget, baselines, metrics, and three seeds. Exact optimization hyperparameters and per-update receipt values should accompany a released configuration before independent replication claims are made.
\begin{table}[h]\centering\caption{Minimum machine-readable CABA evaluation receipt.}\small\begin{tabular}{p{1.3in}p{3.7in}}\toprule Field & Required record\\\midrule Data stream & Dataset version, official split, subject-to-group mapping, and exclusions\\Candidate training & Optimizer, learning rate, batch size, epochs, initialization, and seed\\Anchor split & Split fraction, sampling seed, capacity, replacement rule, and representation version\\Gate & Divergence direction, budget, grid, numerical precision, and batching\\Decision & Accepted coefficient, maximum anchor score, worst-anchor identifier, utility-key result\\Models & Hashes for old, candidate, committed parameters, and frozen backbone\\Metrics & Per-group accuracy matrix, final accuracy, forgetting, uncertainty procedure\\Failures & Rejected updates, infeasible budgets, missing anchors, evaluator errors\\\bottomrule\end{tabular}\end{table}

\section{Recommended Foundation-Model Validation Protocol}\label{app:protocol}
Compare unrestricted commit, STABLE-style gates, and CABA with fixed old policy, candidate, path, and audit data. Cross gates with candidate generators using optimization-time KL or replay; neither is a post-hoc gate for an identical candidate. Match memory across random, coverage-aware, and conflict-aware audits. Separately report end-to-end sequences where accepted histories diverge.

Report utility, historical performance, worst-anchor and average KL, full $g(\alpha)$ traces, accepted coefficients, rejections, reference drift, memory, compute, and wall time. For robots, separate action drift from rollout success and safety events. Predeclare budgets on development streams and hold out the final sequence. These comparisons remain prospective; none is claimed as completed here.

Training-time mechanisms include HiP-LoRA's spectral plasticity budget \citep{chen2026hiplora}; LifeLong-RFT and sequential VLA fine-tuning \citep{liu2026lifelong,hu2026vla}; FOREVER replay \citep{feng2026forever}; trajectory-mixed supervision \citep{khan2026tms}; sparse memory fine-tuning \citep{goyal2026sparse}; and Info-VLA representation constraints \citep{zhao2026infovla}. Gated HAR uses diagonal feature modulation rather than a post-update KL acceptance test \citep{rahimi2026gatedhar}. These are candidate generators or training constraints, not evidence that an external CABA gate improves their outcomes.

\section{Additional Reporting Boundaries}
Supported claims are: (1) CABA verifies a worst-anchor KL constraint under exact arithmetic or a validated numerical error allowance; (2) under explicit coverage and Lipschitz assumptions, the finite check extends to protected support; (3) under the stated MDP assumptions, uniform action KL bounds normalized return degradation; and (4) on UCI-HAR, the selected budget trades accuracy for measured retention.

Unsupported claims are: (1) random anchors cover all safety-relevant states of a foundation-model agent; (2) small KL guarantees semantic correctness, fairness, privacy, or global safety; (3) CABA outperforms replay, STABLE, Memory Anchors, or continual VLA methods; or (4) UCI-HAR predicts LLM, VLA, or robot performance. Trimming or quantile gating requires coverage by the retained subset; otherwise the uniform claim is lost. Reverse, symmetric, or Jensen--Shannon divergence requires matching support assumptions and constants. Label-free auditing remains vulnerable to poisoned observations or stale requirements; quarantine changes coverage. Backbone changes require raw-context reevaluation or a proved representation-error allowance, not merely an alignment score.
\end{document}